\documentclass{ecai}
\usepackage{graphicx}
\usepackage{latexsym}
\usepackage{amsmath,amssymb,amsthm, bm}
\usepackage{fixmath}
\usepackage{bbm}
\usepackage{mathtools}
\usepackage[mathscr]{euscript}

\usepackage{algpseudocode}
\usepackage{algorithm}
\usepackage{multirow}
\usepackage{xcolor}
\usepackage{stfloats}
\usepackage{float}

\newtheorem{thm}[theorem]{Theorem}
\newtheorem{lemma}[theorem]{Lemma}
\newtheorem{proposition}[theorem]{Proposition}

\newtheorem{restatedthminner}{Theorem}
\newenvironment{restatedthm}[1]{%
  \renewcommand\therestatedthminner{#1}%
  \restatedthminner
}{\endrestatedthminner}

\begin{document}

\begin{frontmatter}

\title{Approximate Model-Based Shielding\\ for Safe Reinforcement Learning}

\author[A]{\fnms{Alexander W.}~\snm{Goodall}\thanks{Corresponding Author. Email: a.goodall22@imperial.ac.uk}}
\author[A]{\fnms{Francesco}~\snm{Belardinelli}}

\address[A]{Imperial College London}

\begin{abstract}
Reinforcement learning (RL) has shown great potential for solving complex tasks in a variety of domains. However, applying RL to safety-critical systems in the real-world is not easy as many algorithms are sample-inefficient and maximising the standard RL objective comes with no guarantees on worst-case performance. In this paper we propose approximate model-based shielding (AMBS), a principled look-ahead shielding algorithm for verifying the performance of learned RL policies w.r.t.~a set of given safety constraints. Our algorithm differs from other shielding approaches in that it does not require prior knowledge of the safety-relevant dynamics of the system. We provide a strong theoretical justification for AMBS and demonstrate superior performance to other safety-aware approaches on a set of Atari games with state-dependent safety-labels.
\end{abstract}

\end{frontmatter}

\section{Introduction}

Due to the inefficiencies of deep reinforcement learning (RL) and lack of guarantees, safe RL \cite{amodei2016concrete} has emerged as an increasingly active area of research. Ensuring the safety of RL agents is crucial for their widespread adoption in safety-critical applications where the cost of failure is high. From formal verification, \emph{Shielding} \cite{alshiekh2018safe} has been developed as an approach for safe RL, that comes with strong safety guarantees on the agents performance. However, classical shielding approaches make quite restrictive assumptions which limit their capabilities in real-world and high-dimensional tasks. To this end, we propose approximate model-based shielding (AMBS) to address these limitations and obtain a more general and widely applicable algorithm.

The constrained Markov decision process (CMDP) \cite{altman1999constrained} is a popular way of framing the safe RL paradigm. In addition to maximising reward, the agent must satisfy a set of safety-constraints encoded as a cost function that penalises unsafe actions or states. In effect, this formulation constrains the set of feasible policies, which results in a tricky non-smooth optimisation problem. In high-dimensional settings, several model-free methods have been proposed based on trust-region methods \cite{achiam2017constrained, yangprojection, liu2020ipo} or Lagrangian relaxations of the constrained optimisation problem \cite{ray2019benchmarking, chow2019lyapunov}. Many of these methods rely on the assumption of convergence to obtain any guarantees.

Furthermore, model-based approaches to safe RL have gained increasing traction, in part due to recent significant developments in model-based RL (MBRL) \cite{ha2018world, hafner2020mastering, hafner2023mastering} and the superior sample complexity of model-based approaches \cite{hafner2019dream, janner2019trust}. In addition to learning a reward maximising policy, MBRL learns an approximate dynamics model of the system. Approximating the dynamics using Gaussian process (GP) regression \cite{williams2006gaussian} or ensembles of neural networks are both principled ways to quantify uncertainty, and develop risk- and safety-aware policy optimisation algorithms \cite{thomas2021safe, luo2021learning, as2022constrained} and model predictive control (MPC) schemes \cite{liu2020constrained}. 


Shielding for RL was introduced in \cite{alshiekh2018safe} as a correct by construction reactive system that forces hard constraints on the learned policy, preventing it from entering 
unsafe states determined by a given temporal logic formula. Classical approaches to shielding \cite{alshiekh2018safe} require {\em a priori} access to a safety-relevant abstraction of the environment, where the dynamics are known or at least conservatively estimated. This can be quite a restrictive assumption in many real-world scenarios, where such an abstraction is not known or too complex to represent. However, these strong assumptions do come with the benefit of strong guarantees (i.e., \emph{correctness} and \emph{minimal-interference}). 

In this paper we operate in a less restrictive domain and only assume that there exists some expert labelling of the states; we claim that this is a more realistic paradigm which has been studied in previous work \cite{he2021androids, goodall2023approximate}. As a result, our approach is more broadly applicable to a variety of environments, including high-dimensional environments which have been rarely studied in the shielding literature \cite{odriozola2023shielded}. We do unfortunately loose the strict formal guarantees obtained by classical shielding, although we develop tight probabilistic bounds and a strong theoretical basis for our method.

Our approach, AMBS, is inspired by \emph{latent shielding} \cite{he2021androids} an approximate shielding algorithm that verifies policies in the latent space of a world model \cite{hafner2019dream, hafner2020mastering, hafner2023mastering} and is closely related to previous work on \emph{approximate shielding of Atari agents} \cite{goodall2023approximate}. The core idea of our approach is that we remove the requirement for a suitable safety-relevant abstraction of the environment by learning an approximate dynamics model of the environment. With the learned dynamics model we can simulate possible future trajectories under the learned policy and estimate the probability of committing a safety-violation in the near future. In contrast to previous approaches, we provide a stronger theoretical justification for utilising world models as a suitable dynamics model. And while we leverage DreamerV3 \cite{hafner2023mastering} as our stand-in dynamics model, we propose a more general purpose and model-agnostic framework for approximate shielding. 

\paragraph{Contributions.} Our main contributions are as follows: (1) we formalise the general AMBS framework which captures previous work on latent shielding \cite{he2021androids} and shielding Atari agents \cite{goodall2023approximate}. (2) We formalise notions such as \emph{bounded safety} \cite{giacobbe2021shielding} in Probabilistic Computation-tree Logic (PCTL), a principled specification language for reasoning about the temporal properties of discrete stochastic systems. (3) We provide PAC-style probabilistic bounds on the probability of accurately estimating a constraint violation with both the `true' dynamics model and a learned approximation. (4) We more rigorously develop the theory from \cite{goodall2023approximate} and theoretically justify the use of world models as the stand-in dynamics model. (5) We provide a much richer set of results on a small set of Atari games with state-dependent safety-labels and we demonstrate that AMBS significantly reduces the cumulative safety-violations during training compared to other safety-aware approaches. And in some cases AMBS also vastly improves the learned policy w.r.t.~episode return. All technical proofs are provided in full in the Appendix, along with extended results (learning curves) and implementation details, including hyperparameters and access to code. These materials are also provided online at \url{https://github.com/sacktock/AMBS}.

\section{Preliminaries}

In this section we introduce the relevant background material and notation required to understand our approach and the main theoretical results of the paper. We start by formalising the problem setup and the specification language used to define bounded safety. We then introduce prior look-ahead shielding approaches and world models.

\subsection{Problem Setup}
\label{sec:problem}
In our experiments we evaluate our approach on Atari games provided by the Arcade Learning Environment (ALE) \cite{bellemare13arcade}. Atari games are partially observable, as the agent is only given raw pixel data and does not have access to the underlying memory buffer of the emulator. This means one observation can map to many underlying states by a probability distribution. Therefore, we model the problem as a partially observable markov decision process (POMDP) \cite{puterman1990markov}.

To capture state-dependent safety-labels, we extend the POMDP tuple with a set of atomic propositions and a labelling function; a common formulation used in \cite{baier2008principles}. Formally, we define a POMDP as a 10-tuple $\mathcal{M} = (S, A, p, \iota_{init}, R, \gamma, \Omega, O, AP, L)$, where, $S$ is a finite set of {\em states}, $A$ is a finite set of {\em actions}, $p : S \times A \times S \rightarrow [0, 1]$ is the {\em transition function}, where $p(s'\mid s,a)$ is the probability of transitioning to state $s'$ by taking action $a$ in state $s$, $\iota_{init} : S \rightarrow [0, 1]$ is the {\em initial state distribution}, where $\iota_{init}(s)$ is the probability of starting in state $s$, $R : S \times A \rightarrow \mathbb{R}$ is the {\em reward function}, where $R(s,a)$ is the immediate reward received for taking action $a$ in state $s$, $\gamma \in (0, 1]$ is the discount factor, $\Omega$ is a finite set of {\em observations}, $O : S \times \Omega \rightarrow [0, 1]$ is the {\em observation function}, where $O(o\mid s)$ is the probability of observing $o$ in state $s$, $AP$ is a finite set of {\em atomic propositions} (or atoms),  $L : S \rightarrow 2^{AP}$ is the {\em labeling function}, where $L(s)$ is the set of atoms that hold in state $s$.

\textit{We note here that MDPs are a special case of POMDPs without partial observability. For example, we can suppose that in MDPs, $\Omega = S$ and the observation function $O : S \times \Omega \rightarrow [0, 1]$ collapses to the identity relation.}

In addition, we are given a propositional safety-formula $\Psi$ that encodes the safety-constraints of the environment. For example, a simple $\Psi$ might look like,
\begin{equation*}
    \Psi = \neg \textbf{collision} \land (\textbf{red-light} \Rightarrow \textbf{stop})
\end{equation*}
which says ``don't have a collision and stop when there is a red light'', for $AP = \{ \textbf{collision}, \textbf{red-light}, \textbf{stop}\}$. 

At each timestep $t$, the agent receives an observation $o_t$, reward $r_t$ and a set $L(s_t)$ of labels. The agent can then determine that $s_t$ satisfies the safety-formula $\Psi$ by applying the following relation,
\begin{equation*}
    \begin{array}{@{}r@{{}\mathrel{}}c@{\mathrel{}{}}l@{}}
    s \models a & \text{iff} & a \in L(s)\\
    s \models \neg \Psi & \text{iff} & s \not \models \Psi\\
    s \models \Psi_1 \land \Psi_2 & \text{iff} & s \models \Psi_1 \text{ and } s \models \Psi_2
    \end{array}
\end{equation*}
where $a \in AP$ is an atomic proposition, negation ($\neg$) and conjunction ($\land$) are the familiar logical operators from propositional logic. The goal is to find a policy $\pi$ that maximises reward, that is $\pi^* = \arg\max_{\pi} \mathbb{E}[\sum_{t=0}^{\infty} \gamma^{t} R(s_t, \pi(s_t))]$, while minimising the cumulative number of violations of the safety-formula $\Psi$ during training and deployment.

\subsection{Probabilistic Computation-tree Logic}

Probabilistic Computation Tree Logic (PCTL) is a branching-time temporal logic based on CTL that allows for probabilistic quantification along path formula \cite{baier2008principles}. PCTL is a convenient way of stating soft reachability and safety properties of discrete stochastic systems which makes it a commonly used specification language for probabilistic model checkers. A well-formed PCTL formula can be constructed with the following grammar,
\begin{align*}
    \Phi ::= & \mid a \mid \neg \Phi \mid \Phi \land \Phi \mid \mathbb{P}_{J}(\phi) \\
    \phi ::= & X \Phi \mid \Phi U \Phi \mid \Phi U^{\leq n} \Phi
\end{align*}
where $J \subset [0, 1]$, $J \neq \varnothing$ is a non-empty subset of the unit interval, and next $X$, until $U$, and bounded until $U^{\leq n}$ are the temporal operators from CTL \cite{baier2008principles}. We make the distinction here between state formula $\Phi$ and path formula $\phi$ which are interpreted over states and paths respectively. 

For state formula $\Phi$, we write $s \models \Phi$ to indicate that the state $s \in S$ satisfies the state formula $\Phi$, where the satisfaction relation is defined in a similar way as before (Section \ref{sec:problem}), but also includes probabilistic quantification, see \cite{baier2008principles} for details.  On the other hand, path formula $\phi$ are interpreted over sequences of states or \emph{traces}. In the following section we provide the satisfaction relation for path formula $\phi$ for the fragment of PCTL that we use. We also note that the common temporal operators eventually $\lozenge$ and always $\square$, and their bounded counterparts $\lozenge^{\leq n}$ and  $\square^{\leq n}$ can be derived from the grammar above in a straightforward way, see \cite{baier2008principles} for details.

\subsection{Bounded Safety}

Now we are equipped to formalise the notion of \emph{bounded safety} introduced in \cite{giacobbe2021shielding}. Consider a fixed (stochastic) policy $\pi : O \times A \to [0, 1]$ and a POMDP $\mathcal{M} = (S, A, p, \iota_{init}, R, \gamma, \Omega, O, AP, L)$. Together $\pi$ and $\mathcal{M}$ define a transition system $\mathcal{T} : S \times S \to [0, 1]$ on the set of states $S$, where $\sum_{s' \in S} \mathcal{T}(s'\mid s) = 1$. A {\em finite trace} of the transition system $\mathcal{T}$ with length $n$ ($n$ transitions) is a sequence of states $s_0 \to s_1 \to  \ldots \to s_n$ denoted by $\tau$, where each state transition is induced by the transition probabilities of $\mathcal{T}$. Let $\tau[i]$ denote the $i^\textrm{th}$ state of $\tau$. A trace $\tau$ is said to satisfy bounded safety if all of its states satisfy the given propositional safety-formula $\Psi$, that is,
\begin{eqnarray*}
    \tau \models \square^{\leq n} \Psi & \text{iff} & \forall i \; 0 \leq i \leq n, \tau[i] \models \Psi
\end{eqnarray*}
for some bounded look-ahead parameter $n$. In words, the path formula $\phi = \square^{\leq n} \Psi$ says ``always satisfy $\Psi$ in the next $n$ steps''. Now in PCTL we say that a given state $s \in S$ satisfies $\Delta$-bounded safety or formally $s \models \mathbb{P}_{\geq 1-\Delta}(\square^{\leq n} \Psi )$ iff,
\begin{multline}
   \mu_s(\{\tau \mid \tau[0] = s, \forall i \; 0 \leq i \leq n, \tau[i] \models \Psi\}) \in [1- \Delta, 1] \label{eq:eboundedsafety}
\end{multline}
where $\mu_s$ is a well-defined probability measure (induced by the transition system $\mathcal{T}$) on the set of traces starting from $s$ and with finite length $n$, see \cite{baier2008principles} for additional details on the semantics of PCTL. 

Throughout the rest of the paper we will denote $\mu_{s\models \phi}$ as shorthand for the measure $ \mu_s(\{\tau \mid \tau[0] = s, \forall i \;0 \leq i \leq n, \tau[i] \models \Psi\})$, where $\phi = \square^{\leq n} \Psi$ is the path formula that corresponds to bounded safety. By grounding bounded safety in PCTL we obtain a principled way to trade-off safety and exploration with the $\Delta$ parameter. Since strictly enforcing bounded safety can lead to overly conservative behaviour during training, which is not always desirable if we are to make progress toward the optimal policy.

\subsection{Look-ahead Shielding}

The general principle of look-ahead shielding is as follows, from the current state we compute the set of possible future paths and check if they incur any unsafe states \cite{giacobbe2021shielding}. If the agent proposes an action that leads the agent down a path with an unavoidable unsafe state, then the look-ahead shielding procedure should override the agent's action with an action leading down a safe path instead, see Fig.~\ref{fig:fig1}.

\begin{figure}[h!]
	\centerline{\includegraphics[width=.25\textwidth]{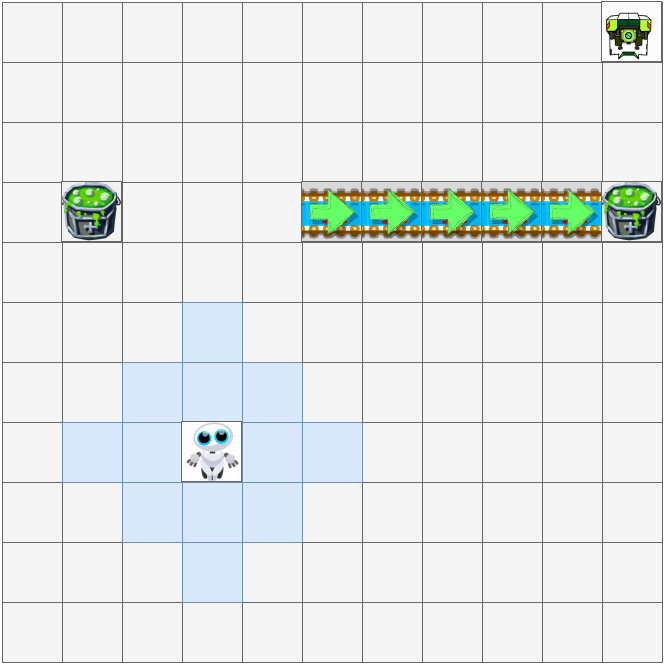}}
	\caption{With a Manhattan distance look-ahead of 2 (blue squares) the robot can avoid the \textit{left-most} vat of acid, but is unable to determine that the conveyor belt leads to an unavoidable unsafe state. Without a further look-ahead horizon, the robot is doomed to fall in the \textit{right-most} vat of acid during exploration. \emph{This image was created with the assistance of DALL·E 2.}}
	\label{fig:fig1}
\end{figure}

Bounded Prescience Shielding (BPS) \cite{giacobbe2021shielding} is an approach to shielding Atari agents that assumes access to a black-box simulator of the environment, which can be queried for look-ahead shielding. In the worst case BPS must enumerate all paths of length $H$ from the current state, to find a safe path for the agent to follow.  This of course comes with a substantial computational overhead, even with a relatively small horizon of $H=5$.

Our approach most closely resembles \emph{latent shielding}, \cite{he2021androids} which rolls out learned policies in the latent space of a \emph{world model} \cite{ha2018world, hafner2019dream} that approximates the `true' dynamics of the environment for look-ahead shielding. By simulating futures in a low-level latent space rather than with a high-fidelity simulator, latent shielding can be used with a much larger look-ahead horizon ($H=15$) and can be deployed during training. As a result we can hope to detect unavoidable unsafe states that require a larger look-ahead horizon to avoid. 

Unfortunately \emph{latent shielding} \cite{he2021androids} is not the be all and end all and is still fairly short-sighted in comparison to our approach. Rolling out dynamics models for long horizons is impractical as we will quickly run into accumulating errors that arise from the approximation error of the learned model. Instead we use \emph{safety critics} for bootstrapping the end of simulated trajectories, to obtain further look-ahead capabilities without running into accumulating errors. 


\subsection{World Models}

\emph{World models} were first introduced in the titular paper \cite{ha2018world}. Built on the recurrent state space model (RSSM) \cite{hafner2019learning}, world models are trained to learn a compact latent representation of the environment state and dynamics. Once the dynamics are learnt the world model can be used to `imagine' possible future trajectories.

We leverage DreamerV3 \cite{hafner2023mastering} as our stand-in dynamics model for look-ahead shielding and policy optimisation. DreamerV3 consists of the following components: (1) the image encoder $z_t \sim q_{\theta}(z_t \mid o_t, h_t)$ that learns a posterior latent representation given the observation $o_t$ and recurrent state $h_t$. (2) The recurrent model $h_t = f_{\theta}(h_{t-1}, z_{t-1},a_{t-1})$, which computes the next deterministic latents given the past state $ s_{t-1} = (h_{t-1}, z_{t-1})$ and action $a_{t-1}$. (3) The transition predictor $\hat z_t \sim p_{\theta}(\hat z_t \mid h_t)$, which is used as the prior distribution over the stochastic latents (without $o_t$). (4) The image decoder $\hat o_t \sim p_{\theta}(\hat o_t \mid h_t, z_t)$ which is used to provide high quality image gradients by minimising reconstruction loss. (5) The prediction heads $\hat r_t \sim p_{\theta}(\hat r_t \mid h_t,z_t )$ and $\hat \gamma_t \sim p_{\theta}(\hat \gamma_t \mid h_t, z_t)$ trained to predict reward signals and episode termination respectively.

In DreamerV3 the latents represent categorical variables and the reward and image prediction heads both parameterise a \emph{twohot symlog} distributions \cite{hafner2023mastering}. All components are implemented as neural networks and optimised with straight through gradients. In addition, KL-balancing \cite{hafner2020mastering} and free-bits \cite{hafner2019dream} are used to regularise the prior and posterior representations and prevent a degenerate latent space representation.

Policy optimisation is performed entirely on experience `imagined' in the latent space of the world model. We refer to $\pi^{\text{task}}$ as the task policy trained in the world model to maximise expected discounted accumulated reward, that is, optimise the following,
\begin{equation}
    \max \mathbb{E}_{\pi^{\text{task}}, p_{\theta}}\left[\sum_{t=1}^{\infty} \hat{\gamma}^{t-1} \hat r_t \right]
    \label{eq:maxrew}
\end{equation}
In addition to the task policy $\pi^{\text{task}}$, a task critic denoted $v^{\text{task}}$, is used to construct TD-$\lambda$ targets \cite{sutton2018reinforcement} which are used to help guide the task policy $\pi^{\text{task}}$ in a TD-$\lambda$ actor-critic style algorithm. We refer the reader to \cite{hafner2023mastering} for more precise details.

\section{Approximate Model-Based Shielding}
\label{sec:ambs}
In this section we present the AMBS framework, our novel method for look-ahead shielding RL policies with a learned dynamics model. We start by giving an overview of the approach, followed by a strong theoretical basis for AMBS. Specifically, we derive PAC-style probabilistic bounds on estimating the probability of a constraint violation under both the `true' transition system and a learned approximation. We show these hold under reasonable assumptions and we demonstrate when these assumptions hold in specific settings. 

\subsection{Overview}

AMBS is split into two main phases: the \emph{learning phase}, which includes (i) dynamics learning, (ii) task policy optimisation with the standard RL objective, 
and (iii) safe policy learning and safety critic optimisation, and the \emph{environment interaction phase}, which consists of collecting experience from the environment with the shielded policy. The new experience is then used to improve the dynamics model. 

To obtain a shielded policy, recall that we are interested in verifying PCTL formulas of the form $\Phi = \mathbb{P}_{\geq 1 -  \Delta}(\square^{\leq n} \Psi )$ ($\Delta$-bounded safety) on the transition system $\mathcal{T}: S \times S \to [0, 1]$ induced by fixing the task policy $\pi^{\textrm{task}}$ for a given POMDP $\mathcal{M}$, where $\Psi$ is the propositional safety-formula. Quite simply, the shielded policy picks actions according to $\pi^{\textrm{task}}$ provided $\Delta$-bounded safety is satisfied, otherwise the shielded policy picks actions with a backup policy that should prevent the agent from committing the possible safety-violation.

To check $\Delta$-bounded safety, we can simulate the transition system $\mathcal{T}$ by rolling-out the learned world model $p_{\theta}$ with actions sampled from the task policy $\pi^{\textrm{task}}$; effectively obtaining an approximate transition system $\widehat{\mathcal{T}} : S \times S \to [0, 1]$. Even with access to the `true' transition system $\mathcal{T}$, exact PCTL model checking is $O(\textrm{poly}(size(\mathcal{T}))\cdot n \cdot |\Phi|)$ which is too big for high-dimensional settings. Rather, we rely on Monte-Carlo sampling from the approximate transition system $\widehat{\mathcal{T}}$ to obtain an estimate $\tilde \mu_{s \models \phi}$ of the measure $\mu_{s\models \phi}$, which specifies the probability of satisfying $\phi =\square^{\leq n} \Psi$ (bounded safety) from $s$, under the `true' transition system $\mathcal{T}$.

\subsection{Fully Observable Setting}

We start by considering the fully observable setting, that is we are concerned with checking PCTL formula of the form $\Phi = \mathbb{P}_{\geq 1 -  \Delta}(\square^{\leq n} \Psi )$ ($\Delta$-bounded safety) on the transition system $\mathcal{T}: S \times S \to [0, 1]$ induced by fixing some given policy $\pi$ in an MDP $\mathcal{M} = (S, A, p, \iota_{init}, R, \gamma, AP, L)$. We state the following result,
\begin{thm}
Let $\epsilon > 0$, $\delta > 0$, $s \in S$ be given. With access to the `true' transition system $\mathcal{T}$, with probability $1 - \delta$ we can obtain an $\epsilon$-approximate estimate of the measure $\mu_{s\models\phi}$, by sampling $m$ traces $\tau \sim \mathcal{T}$, provided that,
\begin{equation}
    m \geq \frac{1}{2\epsilon^2} \log\left(\frac{2}{\delta}\right) \label{eq:boundonm}
\end{equation}
\label{prop:boundonm}
\end{thm}
\begin{proof}
The proof is an application of Hoeffding's inequality \cite{hoeffding1994probability}. See Appendix \ref{sec:proofs} for details.
\end{proof}

Suppose that we only have access to an approximate MDP, where the transition function $p$ is estimated with a learned approximation $\hat p$, that is, $\widehat{\mathcal{M}} = (S, A, \hat p, \iota_{init}, R, \gamma, AP, L)$. As earlier, $\pi$ and $\widehat{\mathcal{M}}$ define an approximate transition system $\widehat{\mathcal{T}}$. We show that we can obtain an estimate $\tilde \mu_{s\models\phi}$ for $\mu_{s\models\phi}$ by sampling traces from $\widehat{\mathcal{T}}$.
\begin{thm}
Let $\epsilon > 0$, $\delta> 0$ be given. Suppose that for all $s \in S$, the total variation (TV) distance between $\mathcal{T}(s' \mid s)$\footnote{Unless $s'$ is fixed we define $\mathcal{T}(s' \mid s)$ as the distribution of next states $s'$ given $s$ rather than as an explicit probability.} and $\widehat{\mathcal{T}}(s' \mid s)$ is bounded by some $\alpha \leq \epsilon/n$. That is,
\begin{equation}
	D_{\text{TV}}\left(\mathcal{T}(s' \mid s), \widehat{\mathcal{T}}(s' \mid s) \right) \leq \alpha \; \forall s \in S 
 \label{eq:tvdistance}
\end{equation}
Now fix an $s \in S$, with probability $1 - \delta$ we can obtain an $\epsilon$-approximate estimate of the measure $\mu_{s\models\phi}$, by sampling $m$ traces $\tau \sim \widehat{\mathcal{T}}$, provided that,
\begin{equation}
    m \geq \frac{2}{\epsilon^2} \log\left(\frac{2}{\delta}\right)
    \label{eq:boundonmbig}
\end{equation}
\label{prop:tv}
\end{thm}
\begin{proof}
The proof follows from Theorem \ref{prop:boundonm} and the \emph{simulation lemma} \cite{kearns2002near} adapted to our purposes, see Appendix \ref{sec:proofs}.
\end{proof}

The assumption we make in Theorem \ref{prop:tv} on the $TV$ distance being bounded for all states $s \in S$ (Eq.~\ref{eq:tvdistance}) is quite strong. In reality we may only require that the $TV$ distance between $\mathcal{T}(s' \mid s)$ and $\widehat{\mathcal{T}}(s' \mid s)$ is bounded for the safety-relevant subset of the state space. 
\paragraph{Tabular Case.} It may also be interesting to consider under what conditions the $TV$ distance between $\mathcal{T}(s' \mid s)$ and $\widehat{\mathcal{T}}(s' \mid s)$ are bounded for a given state. In the tabular case we can obtain an approximate dynamics model $\hat p(s' \mid s, a)$ for the `true' dynamics $p(s' \mid s, a)$ by using the maximum likelihood principle for categorical variables. Quite simply let,
\begin{equation}
    \hat p(s' \mid s, a) = \frac{c(s', s, a)}{v(s, a)}
\end{equation}
where $c(s', s, a)$ is the visit count of $s'$, $s$, $a$ or equivalently the number of times $(s, a)$ has lead $s'$ and $v(s, a)$ is the visit count of $(s, a)$ or similarly the number of times $a$ has been picked from $s$.
\begin{thm}
Let $\alpha > 0$, $\delta > 0$, $s \in S$ be given. With probability $1 - \delta$ the total variation (TV) distance between $\mathcal{T}(s' \mid s)$ and $\widehat{\mathcal{T}}(s' \mid s)$ is upper bounded by $\alpha$, provided that all actions $a \in A$ with non-negligable probability $\eta \geq \alpha/(|A||S|)$ (under $\pi$) have been picked from $s$ at least $m$ times, where
\begin{equation}
	m \geq \frac{|S|^2}{\alpha^2}\log\left( \frac{2 |A||S|}{\delta}\right)
\end{equation}
\label{prop:tabular}
\end{thm}
\begin{proof}
The proof is an application of Hoeffding's inequality \cite{hoeffding1994probability} to categorical distributions, see Appendix \ref{sec:proofs}.
\end{proof}
The strong dependence on $|S|$ here can make this result quite unwieldy. However, for low-rank or low-dimensional MDPs this dependence on $|S|$ can be suitably replaced. For example in \emph{gridworld} environments, where there are only 4 possible next states, $|S|$ can be replaced with 4. Additionally, for deterministic policies we can ignore the dependence on $|A|$.

\subsection{Partially Observable Setting}

Sample complexity bounds like the one we obtained for the tabular MDP case are a lot harder to obtain for general POMDPs. To say anything meaningful about learning in POMDPs, various assumptions can be made about the structure of the environment \cite{lee2023learning}. We will reason about the underlying POMDP $\mathcal{M}$ in terms of \emph{belief states}. Belief states infer a distribution over possible underlying states given the history of past observations and actions. Formally they are defined as a filtering distribution $p(s_t \mid o_{t\leq t}, a_{\leq t})$. Importantly, conditioning on the entire history of past observations and actions exposes more information about the possible underlying states.

However, maintaining an explicit distribution over possible states is difficult, particularly when we have no idea what the space of possible states is. Instead it is common to learn a latent representation $b_t = f(o_{t\leq t}, a_{\leq t})$ of the history of past observations and actions that captures the important statistics of the filtering distribution, so that $p(s_t \mid o_{t\leq t}, a_{\leq t}) \approx p(s_t \mid b_t)$. We will denote $b_t$ the belief state, although it is actually a compact latent state representation of $o_{t\leq t}, a_{\leq t}$, rather than a distribution over possible states.

\begin{thm} \label{prop:pomdp} Let $b_t$ be a latent representation (belief state) such that $p(s_t \mid o_{t\leq t}, a_{\leq t}) = p(s_t \mid b_t)$. Let the fixed policy $\pi(\cdot \mid b_t)$ be a general probability distribution conditional on belief states $b_t$. Let $f$ be a generic $f$-divergence measure (TV or similar). Then the following holds:
\begin{equation}
    D_{f}(\mathcal{T}(s' \mid b ), \widehat{\mathcal{T}}(s' \mid b )) \leq D_{f}(\mathcal{T}(b' \mid b ), \widehat{\mathcal{T}}(b' \mid b ))
\end{equation}
where $\mathcal{T}$ and $\widehat{\mathcal{T}}$ are the `true' and approximate transition system respectively, defined now over both states $s$ and belief states $b$.
\end{thm} 

\begin{proof}
The proof is a direct application of the data-processing inequality \cite{ali1966general} and we note similar bounds have been derived in previous work \cite{gangwani2020learning}. See Appendix \ref{sec:proofs} for details.
\end{proof}

What Theorem \ref{prop:pomdp} says, is that by minimising the divergence between the next belief state $b'$ in the `true' transition system $\mathcal{T}$ and the approximate transition system $\widehat{\mathcal{T}}$ (given the current belief $b$), we minimise an upper bound on the divergence between the next underlying state $s'$ in $\mathcal{T}$ and $\widehat{\mathcal{T}}$. This objective is precisely encoded in the world model loss function of DreamerV3 \cite{hafner2023mastering}. And so by using DreamerV3 (or similar architectures) we aim to minimise the $TV$ distance between $\mathcal{T}$ and $\widehat{\mathcal{T}}$ in the hope that we may use the results stated earlier in this section. 

\section{AMBS with DreamerV3}

In this section we present our practical implementation of AMBS with DreamerV3 \cite{hafner2023mastering}. We start by detailing the important components of the algorithm, before describing the shielding procedure in detail and justifying some of the important algorithmic decisions. An outline for the full procedure is presented in Appendix \ref{sec:algorithms}.


\subsection{Algorithm Components}

\paragraph{RSSM with Costs.} We augment DreamerV3 \cite{hafner2023mastering} with a cost predictor head $\hat c_t \sim p_{\theta}(\hat c_t \mid h_t, z_t)$ used to predict state dependent costs. Similar to the reward head, the cost predictor is also implemented as a neural network that parameterises a \emph{twohot symlog} distribution \cite{hafner2023mastering}. Targets for the cost predictor are constructed as follows,
\begin{equation}
    c_t = \begin{cases}
    0, & \text{if $s_t \models \Psi$} \\
    C, & \text{otherwise}
\end{cases} \label{eq:costtarget}
\end{equation}
where $s_t \models \Psi$ can be determined using the labels $L(s_t)$ provided at each timestep (see Section \ref{sec:problem}), and $C>0$ is a hyperparameter that determines the cost of violating the propositional safety-formula $\Psi$. We claim that using a cost function in this way rather than a binary classifier, allows the agent to distribute its uncertainty about a constraint violation over consecutive states, since the predicted cost at a latent state $(h_t, z_t)$ may lie anywhere on the real number line. 

\paragraph{Safe Policy.} The safe policy denoted $\pi^{\text{safe}}$ is used as the backup policy \cite{hans2008safe} if we detect that a safety-violation is likely to occur in the near future when following the task policy $\pi^{\text{task}}$. The goal of the safe policy is to prevent the agent from committing the safety-violation that was `likely' to happen in the near future. 

Unlike classical shielding methods that synthesise a backup policy ahead of time \cite{alshiekh2018safe}, we must learn a suitable backup policy as we have no knowledge of the safety-relevant dynamics of the system a priori. Although, we note that if such a backup policy is available ahead of time, either because it is easy to determine (i.e. breaking in a car) or significant engineering effort has gone into creating a backup policy, then this can be easily integrated into our framework. However, in general we train a separate policy $\pi^{\text{safe}}$ and critic $v^{\text{safe}}$ using the same TD-$\lambda$ actor-critic style algorithm used for the task policy \cite{hafner2023mastering}. However, we minimise expected discounted costs and ignore any reward signals entirely and so specifically we optimise the following objective,
\begin{equation}
    \min \mathbb{E}_{\pi^{\text{safe}}, p_{\theta}}\left[\sum_{t=1}^{\infty} \hat{\gamma}^{t-1} \hat c_t \right] \label{eq:mincost}
\end{equation}
In theory, by minimising expected costs we should learn a good backup policy, since by construction we only incur a cost $>0$ when $\Psi$ is not satisfied.

\paragraph{Safety Critics.} Safety critics are used to estimate the expected discounted costs under the state distribution of the task policy. In essence, they give us an idea of how safe a given state $s \in S$ 
is if we were to follow the task policy $\pi^{\text{task}}$ from that state $s$. As a result we can use them to bootstrap the end of `imagined' trajectories for further look-ahead capabilities, without having to roll-out the dynamics model any further. We jointly train two safety critics $v^C_1$ and $v^C_2$ with a TD3-style algorithm \cite{fujimoto2018addressing} to prevent overestimation, which is important for reducing overly conservative behaviour. We also train a separate safety discount predictor head $\hat \gamma^{\text{safe}}_t \sim p_{\theta}(\hat \gamma^{\text{safe}}_t \mid h_t, z_t)$ to help with overestimation. Specifically, $v^C_1$ and $v^C_2$ estimate the following quantity,
\begin{equation}
    V^{C}(s) = \mathbb{E}_{\pi^{\text{task}}, p_{\theta}}\left[\sum_{t=1}^{\infty} (\hat{\gamma}^{\text{safe}}_t)^{t-1} \hat c_t \: \Big\vert \: s_0 = s \right] \label{eq:safeyval}
\end{equation}
The safety discount head is a binary classifier that predicts safety-violations, we use it to implicitly transform the MDP into one where violating states are terminal, this ensures that the safety critics are always upper bounded by $C$ and any costs incurred after a safety-violation are not counted. We note that this is an important property to maintain for checking bounded safety.

\subsection{Shielding Procedure} 

During environment interaction we deploy our approximate shielding procedure to try and prevent the task policy $\pi^{\text{task}}$ from violating the safety-formula $\Psi$. Actions that lead to violations in the learned world model $p_{\theta}$ do not count since they are not actually committed in the `real world'.

Algorithm \ref{alg:shield} details the approximate shielding procedure. It takes as input an approximation error $\epsilon$, desired safety level $\Delta$ (from the PCTL formula), the current latent state $\hat s = (z, h)$, an action $a$ proposed by the task policy, the `imagination' horizon $H$, the look-ahead shielding horizon $T$, and the number of samples $m$, along with the RSSM components and the learned policies. The shielding procedure works by sampling $m$ traces from the approximate transition system $\widehat{\mathcal{T}}$ (obtained by sampling actions with $\pi^{\text{task}}$ in $p_{\theta}$), checking whether each trace satisfies bounded safety and returning the proportion of satisfying traces. If the estimate obtained $\tilde \mu_{s \models \phi}$ is in the interval $[1 - \Delta + \epsilon, 1]$ then the proposed action $a$ is verified safe, otherwise we pick an action with the safe policy $\pi^{\text{safe}}$ instead.

\begin{algorithm}[!t]
\caption{Approximate Model-Based Shielding (AMBS)}
\raggedright
\label{alg:shield}
\textbf{Input:} approximation error $\epsilon$, desired safety level $\Delta$, current state $\hat s = (z, h)$, proposed action $a$, `imagination' horizon $H$, look-ahead shielding horizon $T$, number of samples $m$, RSSM $p_{\theta}$, safe policy $\pi^{\text{safe}}$ and task policy $\pi^{\text{task}}$.  \\
\textbf{Output:} shielded action $a'$.
\begin{algorithmic}
\For{$i = 1, ..., m$}
    \State Let $\hat s_0 = (h, z)$
    \State \textit{// Sample trace}
    \State From $\hat s_0$ play $a_0$ and sample trace $\tau = \hat s_{1:H}$ with $\pi^{\text{task}}$ and $p_{\theta}$.
    \State Using $\tau = \hat s_{1:H}$ compute sequences $\hat c_{1:H}$ and $\hat \gamma_{1:H}$ with $p_{\theta}$.
    \State \textit{// Check trace is satisfying}
    \State $X_i = \mathbbm{1}\left[\text{cost}(\tau) < \gamma^{T-1} \cdot C \right]$ using Eq.~\ref{eq:cost} or Eq.~\ref{eq:costcritic}.
\EndFor
\State Let $\tilde \mu_{s \models \phi} = \frac{1}{m}\sum^{m}_{i=1} X_i$
\State \textbf{If} $\tilde \mu_{s \models \phi} \in [1 - \Delta + \epsilon, 1]$ \textbf{return} $a' = a$ \textbf{else} \textbf{return} $a' \sim \pi^{\text{safe}}$
\end{algorithmic}
\end{algorithm}

\begin{proposition}
Suppose we have an estimate $\tilde \mu_{s\models\phi} \in [\mu_{s\models\phi} - \epsilon, \mu_{s\models\phi} + \epsilon]$, if $\tilde \mu_{s\models\phi} \in [1 - \Delta + \epsilon, 1]$ then it must be the case that $\mu_{s\models\phi} \in [1 - \Delta, 1]$ and therefore $s \models \mathbb{P}_{\geq 1 -  \Delta}(\square^{\leq n} \Psi )$.
\label{prop:interval}
\end{proposition}

\begin{proof}
Suppose $\tilde \mu_{s\models\phi} \in [\mu_{s\models\phi} - \epsilon, \mu_{s\models\phi} + \epsilon]$, $\tilde \mu_{s\models\phi} \in [1 - \Delta + \epsilon, 1]$ and $\mu_{s\models\phi} \not\in [1 - \Delta, 1]$. Then $\tilde \mu_{s\models\phi} - \mu_{s\models\phi} > \epsilon $ which contradicts $\tilde \mu_{s\models\phi} \in [\mu_{s\models\phi} - \epsilon, \mu_{s\models\phi} + \epsilon]$. This implies that indeed $\mu_{s\models\phi} \in [1 - \Delta, 1]$ and therefore $s \models \mathbb{P}_{\geq 1 -  \Delta}(\square^{\leq n} \Psi )$ by Eq.~\ref{eq:eboundedsafety}.
\end{proof}

Proposition \ref{prop:interval} justifies checking that the condition $\tilde \mu_{s \models \phi} \in [1 - \Delta + \epsilon, 1]$ holds in Algorithm \ref{alg:shield}. Since if $\tilde \mu_{s \models \phi}$ is an $\epsilon$-estimate of $\mu_{s \models \phi}$ then $\tilde \mu_{s \models \phi} \in [1 - \Delta + \epsilon, 1]$ necessarily implies that the current state $s$ satisfies $\Delta$-bounded safety. We present one more proposition that is required to understand how we determine (or verify) that a trace satisfies bounded safety.
\begin{proposition}
\label{prop:cost}
    Suppose we have a trace $\tau$ with cost given by,
    \begin{equation}
        \text{cost}(\tau) = \sum^H_{t=1} (\hat \gamma_t)^{t-1} \hat c_t 
        \label{eq:cost}
    \end{equation}
    Under the `true' transition system $\mathcal{T}$ if $\textrm{cost}(\tau) < \gamma^{H-1} \cdot C$ then necessarily $\tau \models \square^{\leq H} \Psi$.
\end{proposition}
\begin{proof}
The proof is a straightforward argument. By construction $c_t = C$ if and only if $\tau[t] \not \models \Psi$, therefore $\textrm{cost}(\tau) < \gamma^{H-1} \cdot C$ implies that $\forall t \; 1 \leq t \leq H$ we have $c_t = 0$, which implies that $\forall t \; 1 \leq t \leq H$ we have $\tau[t] \models \Psi$. 
\end{proof}
Furthermore, with safety critics we can compute the cost of a trace in a similar way:
\begin{equation}
        \textrm{cost}(\tau) = \sum^{H-1}_{t=1} (\hat\gamma_t)^{t-1} \hat c_t + \min\left\{v^C_1(h_H, \hat z_t), v^C_2(h_H, \hat z_t)  \right\}
        \label{eq:costcritic}
\end{equation}
Provided that the safety critics accurately capture the expected discounted costs from beyond the `imagination' horizon $H$ then we can check bounded safety with a larger horizon $T > H$ and derive something similar to Proposition \ref{prop:cost}. In this situation it may be clear now why we want the safety critics to be upper bounded by $C$. Since a state $s \in S$ that satisfies bounded safety could have a value $V^C(s) > \gamma^{T-1} \cdot C$ if two or more violations occurred beyond the horizon $T$; further in the future than what we care about.




\section{Experimental Evaluation}

In this section we present the results of running AMBS equipped with DreamerV3 \cite{hafner2023mastering} on a set of Atari games with state-dependent safety-labels \cite{giacobbe2021shielding}. First, we detail the corresponding safety-formula for each of the games. We then cover with the training details, followed by the results and a brief discussion.

\paragraph{Atari Games with Safety-labels.} We evaluate each agent on a set of five Atari games provided in the ALE \cite{bellemare13arcade}. Following \cite{machado18arcade} we use sticky actions (action repeated with probability 0.25), frame skip ($k=4$), and provide no life information unless it is explicitly encoded in the safety-formula. We leverage the state-dependent safety-labels provided in prior work \cite{giacobbe2021shielding} to construct the safety-formula, Table \ref{tab:envs} outlines each of the Atari environments used in our experiments. For a more precise description of each of the environments we refer the reader to \cite{machado18arcade}. We opt for this setting, as opposed to more standard safe RL benchmarks like SafetyGym \cite{ray2019benchmarking}, as it has been used in previous work \cite{giacobbe2021shielding, goodall2023approximate} and it provides a richer set of safety-constraints and corresponding behaviours that much be learnt.
\begin{table}
\begin{center}
{\caption{Atari environments and their corresponding safety-formula}\label{tab:envs}}
\begin{tabular}{c|c}
\hline
\rule{0pt}{8pt}
{\bf Environment} &  {\bf safety-formula} $\Psi$ \\
\hline
Assault & $\neg \textbf{hit} \land \neg \textbf{overheat}$\\
DoubleDunk & $\neg \textbf{out-of-bounds}\land \neg \textbf{shoot-bf-clear}$\\
Enduro & $\neg \textbf{crash-car}$\\
KungFuMaster & $\neg \textbf{loose-life} \land \neg \textbf{energy-loss}$\\
Seaquest & $(\textbf{surface} \Rightarrow \textbf{diver}) \land \neg \textbf{hit} \land \neg \textbf{out-of-oxygen}$\\
\hline
\end{tabular}
\end{center}
\end{table}
\paragraph{Training Details} Each agent is trained on a single Nvidia Tesla A30 (24GB RAM) GPU and a 24-core/48 thread Intel Xeon CPU with 256GB RAM. Due to constraints on compute resources, each agent is trained for 10M environment interactions (40M frames) as opposed to the typical 50M \cite{hafner2020mastering, hafner2023mastering} and our experiments are only run on one seed.

To test robustness, all hyperparameters are fixed over all runs for all agents. The hyperparameters of DreamerV3 for the Atari benchmark are detailed in \cite{hafner2019dream}, most notably the `imagination' horizon $H=15$. The AMBS hyperparameters are also fixed in all experiments as follows, bounded safety level $\Delta = 0.1$, approximation error $\epsilon=0.09$, number of samples $m=512$, and look-ahead shielding horizon $T=30$. Additional implementation details and hyperparameter settings can be found in Appendix \ref{sec:hyperparams}.

\paragraph{Results} We evaluate the effectiveness of AMBS equipped with DreamerV3 \cite{hafner2023mastering} by comparing it with the following algorithms, vanilla DreamerV3 without any shielding procedure or access to the cost function, a safety-aware DreamerV3 that implements an augmented Lagrangian penalty framework (LAG) \cite{wright2006numerical, ray2019benchmarking}, since this method is based on DreamerV3 it should act as a more meaningful baseline compared to similar model-free counterparts like PPO- and TRPO- Lagrangian \cite{ray2019benchmarking} that are not optimised for Atari. 

In addition, we provide results for two state-of-the-art (single GPU) model-free algorithms, Rainbow \cite{hessel2018rainbow} and Implicit Quantile-Network (IQN) \cite{dabney2018implicit}. While these model-free algorithms don't really provide a meaningful baseline, we include them (to the same end as \cite{giacobbe2021shielding}) to demonstrate that state-of-the-art Atari agents frequently violate quite straightforward safety properties that we would expect them to satisfy. We also note that both BPS \cite{giacobbe2021shielding} and classical shielding \cite{alshiekh2018safe} are too unwieldy to be used during training in this setting, and so we cannot obtain a meaningful comparison between these approaches and AMBS.

We compare the five algorithms by recording the cumulative number of violations and the best episode score obtained during a single run of 10M environment interactions (40M frames). Table \ref{tab:results} presents the results. We also provide learning curves for each of the agents in Appendix \ref{sec:curves}.

\begin{table*}[!h]
    \begin{center}
    {\caption{Best episode scores and total violations for the five Atari games. \textit {Blank ($-$) denotes that the agent failed to converge.}}\label{tab:results}}
    \begin{tabular}{cc|ccccc}
    \hline
    \rule{0pt}{12pt}
         & & DreamerV3 & DreamerV3 (AMBS) & DreamerV3 (LAG) & IQN & Rainbow \\
         \hline
         \multirow{2}{6em}{Assault}& Best Score $\uparrow$ &14738 & \textbf{44467} & 19832 & 9959&9632 \\
            & \# Violations $\downarrow$ &18745 &\textbf{12638} &16802 &24462 &24019 \\
        \hline
        \multirow{2}{6em}{DoubleDunk}& Best Score $\uparrow$ & \textbf{24} & \textbf{24} & \textbf{24} & \textbf{24} & - \\
            & \# Violations $\downarrow$ & 877499& \textbf{66248} & 359018 &188363 & - \\
        \hline
        \multirow{2}{6em}{Enduro}& Best Score $\uparrow$ & 2369 &2367 &2365 & 2375 & \textbf{2383} \\
            & \# Violations $\downarrow$ & 167933 & \textit{132147} & 174217 & 129012 & \textbf{108000} \\
        \hline
        \multirow{2}{6em}{KungFuMaster}& Best Score $\uparrow$ & 97000&\textbf{117200} &97200 &51600 &59500 \\
            & \# Violations $\downarrow$ & 427476 &\textbf{10936}  & 567559 & 284909 & 612762 \\
        \hline
        \multirow{2}{6em}{Seaquest}& Best Score $\uparrow$ & 4860 & \textbf{145550} & 1940 & 34150 & 1900 \\
            & \# Violations $\downarrow$ & 73641 & \textbf{40147} & 64679 & 53516 & 67101 \\
        \hline
    \end{tabular}
    \end{center}
\end{table*}

\paragraph{Discussion.}

As we can see in Table \ref{tab:results}, our approach (AMBS) compared with the other DreamerV3 \cite{hafner2023mastering} based agents dramatically reduces the total number of safety-violations during training in all five of the Atari games. In terms of reward, AMBS also achieves comparable or better best episode scores in all five of the Atari games. This is likely because the objective of maximising rewards and minimising constraints are suitably aligned, which we note might not always be the case. LAG fails to reliably reduce the total number of safety-violations in all the Atari games, which could suggest this method is more sensitive to hyperparameter tuning. Clearly the extra machinery introduced by AMBS helps us to obtain a dynamic algorithm that can effectively switch between reward maximising and constraint satisfying policies, which on this set of Atari games demonstrates a significant improvement in performance compared to other safety-aware and state-of-the-art algorithms. Interestingly, the model-free algorithms do better across the board on \emph{Enduro}, it is not immediately clear why this is the case and this result may need further investigation. Perhaps it is the case that DreamerV3 is not quite as suited to Enduro as IQN or Rainbow, although we note that AMBS improves on DreamerV3 w.r.t.~safety-violations while maintaining comparable performance.

\section{Related Work}

\paragraph{Model-based RL.} Model-based RL as a paradigm for learning complex policies has become increasingly popular in recent years due to its superior sample efficiency \cite{hafner2019dream, janner2019trust}. Dyna -- ``an integrated architecture for learning, planning and reacting'' \cite{sutton1991dyna} proposed an architecture that learns a dynamics model of the environment in addition to a reward maximising policy. In theory, by utilising the dynamics model for planning and/or policy optimisation we should learn better policies with fewer experience. Model-based policy optimisation (MBPO) \cite{janner2019trust} is a sample-efficient neural architecture for optimising policies in learned dynamics model. More recently, more sophisticated neural architectures such as Dreamer \cite{hafner2019dream} and DreamerV2 \cite{hafner2020mastering} have been proposed that have demonstrated state-of-the-art performance on the DeepMind Visual Control Suite \cite{tassa2018deepmind} and the Atari benchmark \cite{bellemare13arcade, machado18arcade} respectively. Our work is built on DreamerV3 \cite{hafner2023mastering} which has demonstrated superior performance in a number of domains (with fixed hyperparameters) including the Atari benchmark \cite{bellemare13arcade, machado18arcade} and MineRL \cite{guss2019neurips}, which is regarded as a notoriously difficult exploration challenge in the RL literature.

\paragraph{Shielding.}
Shield synthesis for RL was first introduced in \cite{alshiekh2018safe} as a method for preventing learned policies from entering an unsafe region of the state space defined by a given temporal logic formula. In contrast to other safety-aware approaches based on soft penalties \cite{achiam2017constrained, thomas2021safe}, shielding enforces hard constraints by directly overriding unsafe actions with a verified backup policy. The reactive shield, which includes the backup policy, is computed ahead of time to ensure no training violations (\emph{correctness}). To compute the reactive shield a safety automaton is constructed from a safety-relevant abstraction of the game with known dynamics and a safety game \cite{bloem2015shield} is then solved.

\emph{Look-ahead shielding}  \cite{giacobbe2021shielding, he2021androids, xiao2023model, goodall2023approximate} is a more dynamic approach to shielding that aims to alleviate some of the restrictive requirements that come with classical shielding. The key benefits of look-ahead shielding include, (1) better computational efficiency as only the reachable subset of the state space needs to be checked, (2) can be applied online in real-time from the current state. Although we note that look-ahead shielding can be done ahead of time \cite{xiao2023model} similar to classical shielding. Inspired by \emph{latent shielding} \cite{he2021androids} our approach can very much be categorised as an online (approximate) look-ahead shielding approach.


\section{Conclusions}

In this paper we introduce and describe approximate model-based shielding (AMBS), a general purpose framework for shielding RL policies by simulating and verifying possible futures with a learned dynamics model. In contrast with the majority of work on shielding \cite{alshiekh2018safe}, we apply our approach in a much less restrictive setting, where we only require access to an expert labelling of the states, rather than having the safety-relevant dynamics be known. In addition, we obtain further look-ahead capabilities in comparison to previous work, such as \emph{latent shielding} \cite{he2021androids} and BPS \cite{giacobbe2021shielding}, by utilising safety critics that are trained to predict expected costs. 

Compared with other safety-aware approaches, we empirically show that DreamerV3 \cite{hafner2023mastering} augmented with AMBS demonstrates superior performance in terms of episode return and cumulative safety-violations, on a set of Atari games with state-dependent safety-labels. We also develop a rigorous set of theoretical results that underpin AMBS, which includes strong probabilistic guarantees in the tabular setting. We stress the importance of these results, as it allows operators to meaningfully choose hyperparameter settings based on their specific requirements w.r.t.~safety-guarantees. 

All the components of our algorithm must be learnt from scratch, which unfortunately means that there are few guarantees during early stage training. Although in principle we could jump start the learning process with offline or expert data, this would be an interesting paradigm to explore. Important future work also includes, empirically verifying our theoretical results for the tabular setting and applying AMBS to more standard safe RL benchmarks, like SafetyGym \cite{ray2019benchmarking}, that are not frequently used in the shielding literature. 

\ack This work was supported by UK Research and Innovation [grant number EP/S023356/1], in the UKRI Centre for Doctoral Training in Safe and Trusted Artificial Intelligence (\url{www.safeandtrustedai.org}).

\bibliography{ecai}

\newpage

\appendix
\section{Algorithms}
\label{sec:algorithms}
\subsection{DreamerV3 with AMBS}

\begin{algorithm}[!htb]
\caption{DreamerV3 \cite{hafner2023mastering} with AMBS}
\raggedright
\label{alg:ambs}
\textbf{Initialise:} replay buffer $\mathcal D$ with $S$ random episodes and RSSM parameters $\theta$ randomly.\\
\begin{algorithmic}
\While{not converged}
\State \textit{// World model learning}
\State Sample $B$ sequences $ \{ \langle o_t, a_t, r_t, c_t, \gamma^{\text{safe}}_t, o_{t+1} \rangle^{k+H}_{t=k} \} \sim \mathcal{D}$.
\State Update RSSM parameters $\theta$ with representation learning \cite{hafner2023mastering}.
\State \textit{// Task policy optimisation}
\State `Imagine' trajectory $\sigma$ from every $o_t$ in every seq.~with $\pi^{\text{task}}$.
\State Train $\pi^{\text{task}}$ with TD-$\lambda$ actor-critic to optim.~Eq.~\ref{eq:maxrew}.
\State \textit{// Safety critic optimisation}
\State Train $v^C_1$ and $v^C_2$ with maximum likelihood to estim.~Eq.~\ref{eq:safeyval}.
\State \textit{// Safe policy optimisation}
\State `Imagine' trajectory $\sigma$ from every $o_t$ in every seq.~with $\pi^{\text{safe}}$.
\State Train $\pi^{\text{safe}}$ with TD-$\lambda$ actor-critic to optim.~Eq.~\ref{eq:mincost}.
\State \textit{// Environment interaction}
\For{$k = 1, ..., K$} 
\State Observe $o_t$ from environment and compute $\hat s_t = (z_t, h_t)$.
\State Sample action $a \sim \pi^{\text{task}}$ with the task policy.
\State Shield the proposed action $a' = \text{shield}(a)$ and play $a'$.
\State Observe $r_t, o_{t+1}$ and $L(s_t)$ and construct $c_t$ and $\gamma^{\text{safe}}_t$.
\State Append $\langle o_t, a_t, r_t, c_t, \gamma^{\text{safe}}_t, o_{t+1} \rangle$ to $\mathcal{D}$.
\EndFor
\EndWhile
\end{algorithmic}
\end{algorithm}

\section{Proofs}
\label{sec:proofs}
We provide proofs for the theorems introduced in Section \ref{sec:ambs}.
\subsection{Proof of Theorem \ref{prop:boundonm}}
\begin{restatedthm}{\ref{prop:boundonm} Restated}

Let $\epsilon > 0$, $\delta > 0$, $s \in S$ be given. With access to the `true' transition system $\mathcal{T}$, with probability $1 - \delta$ we can obtain an $\epsilon$-approximate estimate of the measure $\mu_{s\models\phi}$, by sampling $m$ traces $\tau \sim \mathcal{T}$, provided that,

\begin{equation*}
    m \geq \frac{1}{2\epsilon^2} \log\left(\frac{2}{\delta}\right) 
\end{equation*}
\end{restatedthm}

\begin{proof}
    We estimate $\mu_{s\models\phi}$ by sampling $m$ traces $\langle\tau_j\rangle^m_{j=1}$ from $\mathcal{T}$. Let $X_1, ..., X_m$ be indicator r.v.s such that,

\begin{equation*}
    X_j = \begin{cases}
        1 & \text{if $\tau_j \models \square^{\leq n} \Psi$,}\\
        0 & \text{otherwise}
    \end{cases}
\end{equation*}
Let,
\begin{equation*}
    \tilde \mu_{s\models\phi} = \frac{1}{m} \sum^m_{j=1}X_j, \; \text{where} \;\: \mathbb{E}_{\mathcal{T}}[\tilde \mu_{s\models\phi}] = \mu_{s\models\phi} 
\end{equation*}
Then by Hoeffding's inequality,
$$\mathbb{P}\left[|\tilde \mu_{s\models\phi} - \mu_{s\models\phi} | \geq \epsilon \right]\leq 2\exp\left( - 2 m \epsilon^2 \right)$$
Bounding the RHS from above with $\delta$ and rearranging completes the proof.
\end{proof}
The only caveat is arguing that $\tau_j \models \square^{\leq n} \Psi$ is easily checkable. Indeed this is the case (for polynmial $n$) because in the fully observable setting we have access to the state and so we can check that $\forall i \; \tau[i] \models \Psi$.

\subsection{Proof of Theorem \ref{prop:tv}}

We split the proof into two parts for Theorem \ref{prop:tv}, first we present the \emph{error amplification} lemma \cite{rajeswaran2020game}, followed by the full proof of Theorem \ref{prop:tv}.

\begin{lemma}[error amplification] Let $\mathcal{T}(s' \mid s)$ and $\widehat{\mathcal{T}}(s' \mid s)$ be two transition systems with the same initial state distribution $\iota_{init}$. Let $\mathcal{T}^t(s)$ and $\widehat{\mathcal{T}}^t(s)$ be the marginal state distribution at time $t$ for the transitions systems $\mathcal{T}$ and $\widehat{\mathcal{T}}$ respectively. That is,
\begin{eqnarray*}
    \mathcal{T}^t(s) = \mathbb{P}_{\tau \sim \mathcal{T}}[\tau[t] = s] \\
    \widehat{\mathcal{T}}^t(s) = \mathbb{P}_{\tau \sim \widehat{\mathcal{T}}}[\tau[t] = s]
\end{eqnarray*}
Suppose that,
\begin{equation*}
	D_{\text{TV}}\left(\mathcal{T}(s' \mid s), \widehat{\mathcal{T}}(s' \mid s) \right) \leq \alpha \; \forall s \in S
\end{equation*}
then the marginal distributions are bounded as follows,
\begin{equation*}
	D_{\text{TV}}\left(\mathcal{T}^t, \widehat{\mathcal{T}}^t \right) \leq \alpha t \; \forall t
\end{equation*}
\label{lemma:marginal}
\end{lemma}
\begin{proof}
First let us fix some $s \in S$. Then,
\begin{align*}
    \big\vert\mathcal{T}^t(s) - \widehat{\mathcal{T}}^t(s)\big\vert & = \Big\vert \sum_{\bar s \in S} \mathcal{T}(s \mid \bar s) \mathcal{T}^{t-1}(\bar s) - \sum_{\bar s \in S} \widehat{\mathcal{T}}(s \mid \bar s) \widehat{\mathcal{T}}^{t-1}(\bar s)\Big\vert \\
    & \leq \sum_{\bar s \in S} \big\vert \mathcal{T}(s \mid \bar s) \mathcal{T}^{t-1}(\bar s) - \widehat{\mathcal{T}}(s \mid \bar s) \widehat{\mathcal{T}}^{t-1}(\bar s) \big\vert \\
    & \leq \sum_{\bar s \in S} \big\vert \mathcal{T}^{t-1}(\bar s)\left( \mathcal{T}(s \mid \bar s) - \widehat{\mathcal{T}}(s \mid \bar s) \right) \big\vert \\
    & \qquad + \big\vert \widehat{\mathcal{T}}(s \mid \bar s)\left( \mathcal{T}^{t-1}(\bar s) - \widehat{\mathcal{T}}^{t-1}(\bar s)\right)\big\vert
\end{align*}    
Using the above inequality we get the following,
\begin{align*}
    2 D_{\text{TV}}\left(\mathcal{T}^t, \widehat{\mathcal{T}}^t \right) & = \sum_{s \in S} \big\vert\mathcal{T}^t(s) - \widehat{\mathcal{T}}^t(s)\big\vert \\
    & \leq \sum_{\bar s \in S} \mathcal{T}^{t-1}(\bar s) \sum_{s \in S} \big\vert \mathcal{T}(s \mid \bar s) - \widehat{\mathcal{T}}(s \: |\: \bar s) \big\vert \\
    & + \sum_{\bar s \in S} \big\vert \mathcal{T}^{t-1}(\bar s) - \widehat{\mathcal{T}}^{t-1}(\bar s) \big\vert \\
    & \leq 2 \alpha + 2 D_{\text{TV}}\left(\mathcal{T}^{t-1}, \widehat{\mathcal{T}}^{t-1} \right) \\
    & \leq 2 \alpha t
\end{align*}
The final inequality holds by applying the the recursion obtained on $t$ until $t = 0$ where $\mathcal{T}$ and $\widehat{\mathcal{T}}$ start from the same initial state distribution $\iota_{init}$.
\end{proof}
\begin{restatedthm}{\ref{prop:tv} Restated}
Let $\epsilon > 0$, $\delta> 0$ be given. Suppose that for all $s \in S$, the total variation (TV) distance between $\mathcal{T}(s' \mid s)$ and $\widehat{\mathcal{T}}(s' \mid s)$ is bounded by some $\alpha \leq \epsilon/n$. That is,
\begin{equation*}
	D_{\text{TV}}\left(\mathcal{T}(s' \mid s), \widehat{\mathcal{T}}(s' \mid s) \right) \leq \alpha \; \forall s \in S
\end{equation*}
Now fix an $s \in S$, with probability $1 - \delta$ we can obtain an $\epsilon$-approximate estimate of the measure $\mu_{s\models\phi}$, by sampling $m$ traces $\tau \sim \widehat{\mathcal{T}}$, provided that,
\begin{equation*}
    m \geq \frac{2}{\epsilon^2} \log\left(\frac{2}{\delta}\right) 
\end{equation*}
\end{restatedthm}
\begin{proof}
Recall that,
\begin{equation*}
    \mu_{s \models \phi} = \mu_s(\{\tau \mid \tau[0] = s, \text{ for all }0 \leq i \leq n, \tau[i] \models \Psi\})
\end{equation*}
where $\tau \sim \mathcal{T}$. Equivalently we can write,
\begin{equation*}
    \mu_{s \models \phi} = \mathbb{P}_{\tau \sim \mathcal{T}}[\tau \models \square^{\leq n} \Psi ]
\end{equation*}
Similarly, let $\hat \mu_{s \models \phi}$ be defined as the \emph{true} probability under $\widehat{\mathcal{T}}$,
\begin{equation*}
    \hat \mu_{s \models \phi} = \mathbb{P}_{\tau \sim \widehat{\mathcal{T}}}[\tau \models \square^{\leq n} \Psi]
\end{equation*}
Let the following denote the average state distribution for $\mathcal{T}$ and $\widehat{\mathcal{T}}$ respectively,
\begin{eqnarray*}
    \rho_{\mathcal{T}}(s) = \frac{1}{n}\sum^n_{i=1}\mathbb{P}_{\tau \sim \mathcal{T}}(\tau[i] = s) \\
    \rho_{\widehat{\mathcal{T}}}(s) = \frac{1}{n}\sum^n_{i=1}\mathbb{P}_{\tau \sim \widehat{\mathcal{T}}}(\tau[i] = s) \\
\end{eqnarray*}
By following the simulation lemma \cite{kearns2002near}, we get the following,
\begin{align*}
    \big\vert \mu_{s \models \phi} - \hat \mu_{s \models \phi} \big\vert & = \big\vert \mathbb{P}_{\tau \sim \mathcal{T}}[\tau \models \square^{\leq n} \Psi ] - \mathbb{P}_{\tau \sim \widehat{\mathcal{T}}}[\tau \models \square^{\leq n} \Psi ] \big\vert \\
    & \leq  1 \cdot D_{TV}(\rho_{\mathcal{T}}, \rho_{\widehat{\mathcal{T}}})\\
    &  = \frac{1}{2}\sum_{s \in S} \big\vert \rho_{\mathcal{T}}(s) - \rho_{\widehat{\mathcal{T}}}(s) \big\vert \\
    & = \frac{1}{2n}\sum_{s \in S} \Big\vert \sum^{n}_{i=1}\mathbb{P}_{\tau \sim \mathcal{T}}(\tau[i] = s) -  \mathbb{P}_{\tau \sim \widehat{\mathcal{T}}}(\tau[i] = s) \Big\vert \\
    & \leq \frac{1}{2n}\sum_{s \in S} \sum^{n}_{i=1} \Big\vert \mathbb{P}_{\tau \sim \mathcal{T}}(\tau[i] = s) -  \mathbb{P}_{\tau \sim \widehat{\mathcal{T}}}(\tau[i] = s) \Big\vert \\
    & \leq \frac{1}{2n} \sum^{n}_{i=1} \alpha n  \qquad \text{(Using Lemma \ref{lemma:marginal})}\\
    & = \frac{\alpha n}{2}
\end{align*}
Now we have that $\big\vert \mu_{s \models \phi} - \hat \mu_{s \models \phi} \big\vert \leq (\alpha n)/2 \leq \epsilon/2$. It remains to obtain an $\epsilon/2$-approximation of $\hat \mu_{\models \phi}$. Using the exact same reasoning as in the proof of Theorem \ref{prop:boundonm}, we estimate $\hat \mu_{\models \phi}$ by sampling $m$ traces $\langle\tau_j\rangle^m_{j=1}$ from $\widehat{\mathcal{T}}$. Then provided,
\begin{equation*}
    m \geq \frac{2}{\epsilon^2} \log\left(\frac{2}{\delta}\right) 
\end{equation*}
with probability $1 - \delta$ we obtain an $\epsilon/2$-approximation of $\hat \mu_{\models \phi}$ and by extension an $\epsilon$-approximation of $\mu_{\models \phi}$.
\end{proof}

\subsection{Proof of Theorem \ref{prop:tabular}}
\begin{restatedthm}{\ref{prop:tabular} Restated}
Let $\alpha > 0$, $\delta > 0$, $s \in S$ be given. With probability $1 - \delta$ the total variation (TV) distance between $\mathcal{T}(s' \mid s)$ and $\widehat{\mathcal{T}}(s' \mid s)$ is upper bounded by $\alpha$, provided that all actions $a \in A$ with non-negligable probability $\eta \geq \alpha/(|A||S|)$ (under $\pi$) have been picked from $s$ at least $m$ times, where
\begin{equation*}
	m \geq \frac{|S|^2}{\alpha^2}\log\left( \frac{2 |S||A|}{\delta}\right)
\end{equation*}
\end{restatedthm}
\begin{proof}
Recall that $p(s' \mid s, a)$ denotes the probability of transitioning to $s'$ from $s$ when action $a$ is played. Lets first fix $s'$ and just consider approximating one of these probabilities. Let,
\begin{itemize}
    \item $p = p(s' \mid s, a)$ and
    \item $m$ be the number of times $a$ is played from $s$.
    \item Let $X_1, ..., X_m$ be the indicator r.v.s such that,
    \begin{equation*}
        X_{i} = \begin{cases}
        1 & \text{if $s'$ given $(s, a)$}\\
        0 & \text{otherwise}
    \end{cases}
    \end{equation*}
    \item $\hat p = \frac{1}{m} \sum^m_{i=1} X_i$, where $\mathbb{E}[\hat p] = p$.
\end{itemize}
By Hoeffding's Inequality we have,
\begin{equation*}
    \mathbb{P}\left[ |\hat p - p| \geq \frac{\alpha}{|S|}\right] \leq 2 \cdot \exp \left( -2 m \frac{\alpha^2}{|S|^2}\right)
\end{equation*}
Bounding the RHS from above by $\delta / (|S||A|)$ gives us,
\begin{equation*}
    m \geq \frac{|S|^2}{\alpha^2} \log \left(\frac{2|S||A|}{\delta} \right)
\end{equation*}
And so with probability $1 - \delta/(|S||A|)$ we have an $\alpha/|S|$-approximation for $p$ provided $m$ satisfies the above bound. Taking a union bound over all $s' \in S$ and all $a \in A$ we have with probability at least $1 - \delta$ an $\alpha/|S|$-approximation $p(s' \mid s, a)$ for all $s' \in S$ and all actions $a \in A$ with non-negligible probability $\eta \geq \frac{\alpha}{|A|}$ (under $\pi$). It remains to show that the TV distance between $\mathcal{T}(s' \mid s)$ and $\widehat{\mathcal{T}}(s' \mid s)$ is upper bounded by $\alpha$,
\begin{align*}
    &2 D_{TV} \left(\mathcal{T}(s' \mid s),  \widehat{\mathcal{T}}(s' \mid s)\right) \\
     &= \sum_{s' \in S} \big\vert\mathcal{T}(s' \mid s) - \widehat{\mathcal{T}}(s' \mid s) \big\vert \\
     &= \sum_{s' \in S} \sum_{a \in A} \big\vert p(s' \mid s, a)\pi(a \mid s) - \hat p(s' \mid s, a)\pi(a \mid s) \big\vert \\
     &= \sum_{s' \in S} \Bigg( \sum_{a \in A : \pi(a \mid s) \geq \eta} \big\vert p(s' \mid s, a)\pi(a \mid s) - \hat p(s' \mid s, a)\pi(a \mid s) \big\vert \\
      & \qquad + \sum_{a \in A : \pi(a \mid s) < \eta} \big\vert p(s' \mid s, a)\pi(a \mid s) - \hat p(s' \mid s, a)\pi(a \mid s) \big\vert \Bigg) \\ 
      & = \sum_{s' \in S} \Bigg( \sum_{a \in A : \pi(a \mid s) \geq \eta} \pi(a \mid s) \big\vert p(s' \mid s, a) - \hat p(s' \mid s, a) \big\vert \\
      & \qquad + \sum_{a \in A : \pi(a \mid s) < \eta} \pi(a \mid s) \big\vert p(s' \mid s, a) - \hat p(s' \mid s, a) \big\vert \Bigg)\\
      & \leq \sum_{s' \in S} \Bigg(\sum_{a \in A : \pi(a \mid s) \geq \eta} \pi(a \mid s) \frac{\alpha}{|S|} + \sum_{a \in A : \pi(a \mid s) < \eta} \pi(a \mid s) \Bigg) \\
      &\leq \sum_{s' \in S} \left( \frac{\alpha}{|S|} + |A| \cdot \eta \right) \\
      & \leq \sum_{s' \in S} \left(\frac{\alpha}{|S|} + \frac{\alpha}{|S|}\right) \\
      &\leq 2 \alpha
\end{align*}
\end{proof}

\subsection{Proof of Theorem \ref{prop:pomdp}}

\begin{restatedthm}{\ref{prop:pomdp} Restated} Let $b_t$ be a latent representation (belief state) such that $p(s_t \mid o_{t\leq t}, a_{\leq t}) = p(s_t \mid b_t)$. Let the fixed policy $\pi(\cdot \mid b_t)$ be a general probability distribution conditional on belief states $b_t$. Let $f$ be a generic $f$-divergence measure (TV or similar). Then the following holds:
\begin{equation*}
    D_{f}(\mathcal{T}(s' \mid b ), \widehat{\mathcal{T}}(s' \mid b )) \leq D_{f}(\mathcal{T}(b' \mid b ), \widehat{\mathcal{T}}(b' \mid b ))
\end{equation*}
where $\mathcal{T}$ and $\widehat{\mathcal{T}}$ are the `true' and approximate transition system respectively, defined now over both states $s$ and belief states $b$.
\end{restatedthm} 

\begin{proof}
First for clarity, we define the transitions systems $\mathcal{T}$ and $\widehat{\mathcal{T}}$ over states $s$ and belief states $b$. First we have as before
\begin{eqnarray*}
    \mathcal{T}(s' \mid b) = \mathbb{P}_{\pi, p}[s_t = s' \mid b_{t-1} = b] \\
    \widehat{\mathcal{T}}(s' \mid b) = \mathbb{P}_{\pi, \hat p}[s_t = s' \mid b_{t-1} = b]
\end{eqnarray*}
Additionally, let,
\begin{eqnarray*}
    \mathcal{T}(b' \mid b) = \mathbb{P}_{\pi, p}[b_t = b' \mid b_{t-1} = b] \\
    \widehat{\mathcal{T}}(b' \mid b) = \mathbb{P}_{\pi, \hat p}[b_t = b' \mid b_{t-1} = b]
\end{eqnarray*}
From these definitions, note that we can immediately define conditional and joint probabilities (i.e. $\mathcal{T}(s', b' \mid b)$, $\mathcal{T}(s' \mid b)$, $\mathcal{T}(b' \mid s', b)$ and similarly for $\widehat{\mathcal{T}}$) using the standard laws of probability. 

Now we are ready to apply the data-processing inequality \cite{ali1966general} for f-divergences as follows,
\begin{align*}
    &D_{f}(\mathcal{T}(b' \mid b), \widehat{\mathcal{T}}(b' \mid b )) \\  
    &= \mathbb{E}_{b' \sim \widehat{\mathcal{T}}}\left[ f \left( \frac{\mathcal{T}(b' \mid b)}{\widehat{\mathcal{T}}(b' \mid b)}\right) \right] \\
    & = \mathbb{E}_{s', b' \sim \widehat{\mathcal{T}}}\left[ f \left( \frac{\mathcal{T}(s', b' \mid b)}{\widehat{\mathcal{T}}(s', b' \mid b)}\right) \right] \\
    & = \mathbb{E}_{s' \sim \widehat{\mathcal{T}}}\left[  \mathbb{E}_{b' \sim \widehat{\mathcal{T}}} f \left( \frac{\mathcal{T}(s', b' \mid b)}{\widehat{\mathcal{T}}(s', b' \mid b)}\right) \right] \\
    & \geq \mathbb{E}_{s' \sim \widehat{\mathcal{T}}}\left[ f \left( \mathbb{E}_{b' \sim \widehat{\mathcal{T}}}  \frac{\mathcal{T}(s', b' \mid b)}{\widehat{\mathcal{T}}(s', b' \mid b)}\right) \right] \qquad \text{(Jensen's)} \\
    & = \mathbb{E}_{s' \sim \widehat{\mathcal{T}}}\left[ f \left( \mathbb{E}_{b' \sim \widehat{\mathcal{T}}}  \frac{\mathcal{T}(s', b' \mid b) \widehat{\mathcal{T}}(b' \mid s', b)}{\widehat{\mathcal{T}}(s', b' \mid b)\mathcal{T}(b' \mid s', b)}\right) \right] \\
    & = \mathbb{E}_{s' \sim \widehat{\mathcal{T}}}\left[ f \left( \mathbb{E}_{b' \sim \widehat{\mathcal{T}}}  \frac{\mathcal{T}(s' \mid b)}{\widehat{\mathcal{T}}(s' \mid b)}\right) \right] \\
    & = \mathbb{E}_{s' \sim \widehat{\mathcal{T}}}\left[ f \left( \frac{\mathcal{T}(s' \mid b)}{\widehat{\mathcal{T}}(s' \mid b)}\right) \right] \\
    & = D_{f}(\mathcal{T}(s' \mid b ), \widehat{\mathcal{T}}(s' \mid b ))
\end{align*}
\end{proof}

\section{Code and Hyperparameters}
\label{sec:hyperparams}
Here we specify the most important hyperparameters for each of the agents in our experiments. For more precise implementation details we refer the reader to \url{https://github.com/sacktock/AMBS}.

\paragraph{World Model.} As discussed we leverage DreamerV3 \cite{hafner2023mastering}. For all architectural details and hyperparameter choices please refer to \cite{hafner2023mastering}.

\paragraph{Predictor Heads} The cost function and safety discount predictor heads are implemented as neural network architectures similar to those used for the reward predictor and termination predictor of DreamerV3 \cite{hafner2023mastering}. Specifically the cost function predictor head is implemented in exactly the same was as the reward predictor head except it is used to predict cost signals instead of reward signals. Similarly, the safety discount predictor is a binary classifier implemented in the exact same was as the termination predictor, except one predicts safety-violations and the other predicts episode termination.

\paragraph{Safe Policy} As discussed, the safe policy is trained with the same TD-$\lambda$ style actor-critic algorithm that is used for the standard reward maximising (task) policy of DreamerV3 \cite{hafner2023mastering}. The actor and critic are implemented as neural networks with the same architecture that is used for the task policy and the hyperparameters are consistent between the safe policy and the task policy. Other details are outlined in Table \ref{tab:dreamerv3}.

\paragraph{Safety Critics} The safety critics are trained with a TD3-style algorithm \cite{fujimoto2018addressing}. The two critics themselves (and their target networks) are implemented as neural networks with the same architectures used for the critics that help train the task policy and safe policy. The hyperparameters also are mostly the same and any changes are outlined in Table \ref{tab:ambs}.
\subsection{DreamerV3 Hyperparameters}
\begin{table}
    \begin{center}
    {\caption{DreamerV3 hyperparameters \cite{hafner2023mastering}. Other methods built on DreamerV3 such as AMBS and LAG use this set of hyperparameters as well, unless otherwise specified.}
    \label{tab:dreamerv3}}   
    \begin{tabular}{c|c|c}
    \hline
       Name  & Symbol & Value \\
       \hline
       \multicolumn{3}{c}{General}\\
       \hline
         Replay capacity& - & $10^6$\\
         Batch size & $B$ & 16\\
         Batch Length & - & 64\\
         Num. Envs & - & 8 \\
         Train ratio & - & 64 \\
         MLP Layers & - & 5 \\
         MLP Units & - & 512 \\
         Activation & - & LayerNorm + SiLU \\
         \hline
         \multicolumn{3}{c}{World Model}\\
         \hline
         Num. latents & - & 32\\
         Classes per latent & - & 32\\
         Num. Layers & - & 5 \\
         Num. Units & - & 1024 \\
         Recon. loss scale & $\beta_{\text{pred}}$ & 1.0\\
         Dynamics loss scale & $\beta_{\text{dyn}}$ & 0.5\\
         Represen. loss scale & $\beta_{\text{rep}}$ & 0.1\\
         Learning rate & - & $10^{-4}$ \\
         Adam epsilon & $\epsilon_{\text{adam}}$& $10^{-8}$\\
         Gradient clipping & - & 1000 \\
         \hline
         \multicolumn{3}{c}{Actor Critic}\\
         \hline
         Imagination horizon & H & 15\\
         Discount factor & $\gamma$ & 0.997 \\
         TD lambda & $\lambda$ & 0.95\\
         Critic EMA decay & - & 0.98 \\
         Critic EMA regulariser & -&1\\
         Return norm. scale & $S_{\text{reward}}$ & $\text{Per}(R, 95) - \text{Per}(R, 5) $ \\
         Return norm. limit & $L_{\text{reward}}$ & 1\\
         Return norm. decay & - & 0.99\\
         Actor entropy scale & $\eta_{\text{actor}}$ & $3 \cdot 10^{-4}$\\
         Learning rate & - & $3 \cdot 10^{-5}$\\
         Adam epsilon & $\epsilon_{\text{adam}}$& $10^{-5}$\\
         Gradient clipping & - & 100 \\
         \hline
    \end{tabular}
    \end{center}
\end{table}
\clearpage
\subsection{AMBS Hyperparameters}
\begin{table}
    \begin{center}
    {\caption{AMBS hyperparameters. We note that $m>512$ is sufficient for $\Delta = 0.1$, $\epsilon=0.09$, $\delta =0.01$ using a bound similar to Eq.~\ref{eq:boundonmbig} that gives a bound on overestimating $\mu_{s \models \phi}$.}
    \label{tab:ambs}}
    \begin{tabular}{c|c|c}
    \hline
       Name  & Symbol & Value \\
       \hline
       \multicolumn{3}{c}{Shielding}\\
       \hline
         Safety level& $\Delta$ & 0.1\\
         Approx. error & $\epsilon$ & 0.09\\
         Num. samples & $m$ & 512\\
         Failure probability & $\delta$ & 0.01 \\
         Look-ahead horizon & $T$ & 30\\
         Cost Value & $C$ & 10\\
         \hline
         \multicolumn{3}{c}{Safe Policy}\\
         \hline
         \multicolumn{3}{c}{See `Actor Critic' table \ref{tab:dreamerv3}} \\
         \multicolumn{3}{c}{...} \\
         \hline
         \multicolumn{3}{c}{Safety Critic}\\
         \hline
         Type & - & TD3-style \cite{fujimoto2018addressing}\\
         Slow update freq. & - & 1\\
         Slow update fraction & - & 0.02\\
         EMA decay & - & 0.98 \\
         EMA regulariser & -&1\\
         Cost norm. scale & $S_{\text{cost}}$ & $\text{Per}(R, 95) - \text{Per}(R, 5) $ \\
         Cost norm. limit & $L_{\text{cost}}$ & 1\\
         Cost norm. decay & - & 0.99\\
         Learning rate & - & $3 \cdot 10^{-5}$\\
         Adam epsilon & $\epsilon_{\text{adam}}$& $10^{-5}$\\
         Gradient clipping & - & 100 \\
         \hline
    \end{tabular}
    \end{center}
\end{table}
\subsection{LAG Hyperparameters}
\begin{table}
    \begin{center}
    {\caption{LAG hyperparameters \cite{ray2019benchmarking}. We use the default hyperparameters provided in \cite{as2022constrained} for the safety gym benchmark \cite{ray2019benchmarking}. }
    \label{tab:lcpo}}
    \begin{tabular}{c|c|c}
    \hline
       Name  & Symbol & Value \\
       \hline
        Penalty Multiplier & $\mu_k$ & $5 \cdot 10^{-9}$ \\
        Lagrange Multiplier & $\lambda^k$ & $10^{-6}$ \\
        Penalty Power & $\sigma$ & $10^{-5}$ \\
        Safety Horizon & $T$ & 30 \\
        Cost Value & $C$ & 10 \\
        Cost Threshold & $d = C \cdot \gamma^{T}$ & $\approx 9$ \\
        \hline
    \end{tabular}
    \end{center}
\end{table}
\newpage
\subsection{IQN Hyperparameters}
\begin{table}
    \begin{center}
    {\caption{IQN hyperparameters \cite{dabney2018implicit} adapted for a fairer comparison as in \cite{castro18dopamine}.}
    \label{tab:iqn}}
    \begin{tabular}{c|c|c}
    \hline
       Name  & Symbol & Value \\
       \hline
        Replay capacity & - & $10^{6}$\\
        Batch size & $B$ & 32 \\
        IQN kappa & $\kappa$ & 1.0\\
        Num. $\tau$ samples & - & 64\\
        Num. $\tau'$ samples & - & 64\\
        Num. quantile samples & - & 32\\
        Discount factor & $\gamma$ & $0.99$ \\
        Update horizon & $n_{\text{hor}}$ & 3 \\
        Update freq. & - & $4$ \\
        Target update freq. & - & 8000 \\
        Epsilon greedy min & $\epsilon_{\text{min}}$ & 0.01 \\
        Epsilon decay period & - & 250000 \\
        Optimiser & - & Adam\\
        Learning rate & - & $5 \cdot 10^{-5}$\\
        Adam epsilon & $\epsilon_{\text{adam}}$ & $3.125 \cdot 10^{-4}$\\
        \hline
    \end{tabular}
    \end{center}
\end{table}
\subsection{Rainbow Hyperparameters}
\begin{table}
    \begin{center}
    \caption{Rainbow hyperparameters as recommended in \cite{hessel2018rainbow}.}
    \label{tab:rainbow}   
    \begin{tabular}{c|c|c}
    \hline
       Name  & Symbol & Value \\
       \hline
        Replay capacity & - & $10^{6}$\\
        Batch size & $B$ & 32 \\
        Discount factor & $\gamma$ & $0.99$ \\
        Update horizon & $n_{\text{hor}}$ & 3 \\
        Update freq. & - & $4$ \\
        Target update freq. & - & 8000 \\
        Epsilon greedy min & $\epsilon_{\text{min}}$ & 0.01 \\
        Epsilon decay period & - & 250000 \\
        Optimiser & - & Adam\\
        Learning rate & - & $6.25 \cdot 10^{-5}$\\
        Adam epsilon & $\epsilon_{\text{adam}}$ & $1.5 \cdot 10^{-4}$\\
        \hline
    \end{tabular}
    \end{center}
\end{table}

\newpage
\clearpage
\section{Learning Curves}
\label{sec:curves}
    \begin{figure*}[b]
	\centerline{\includegraphics[width=1.0\textwidth]{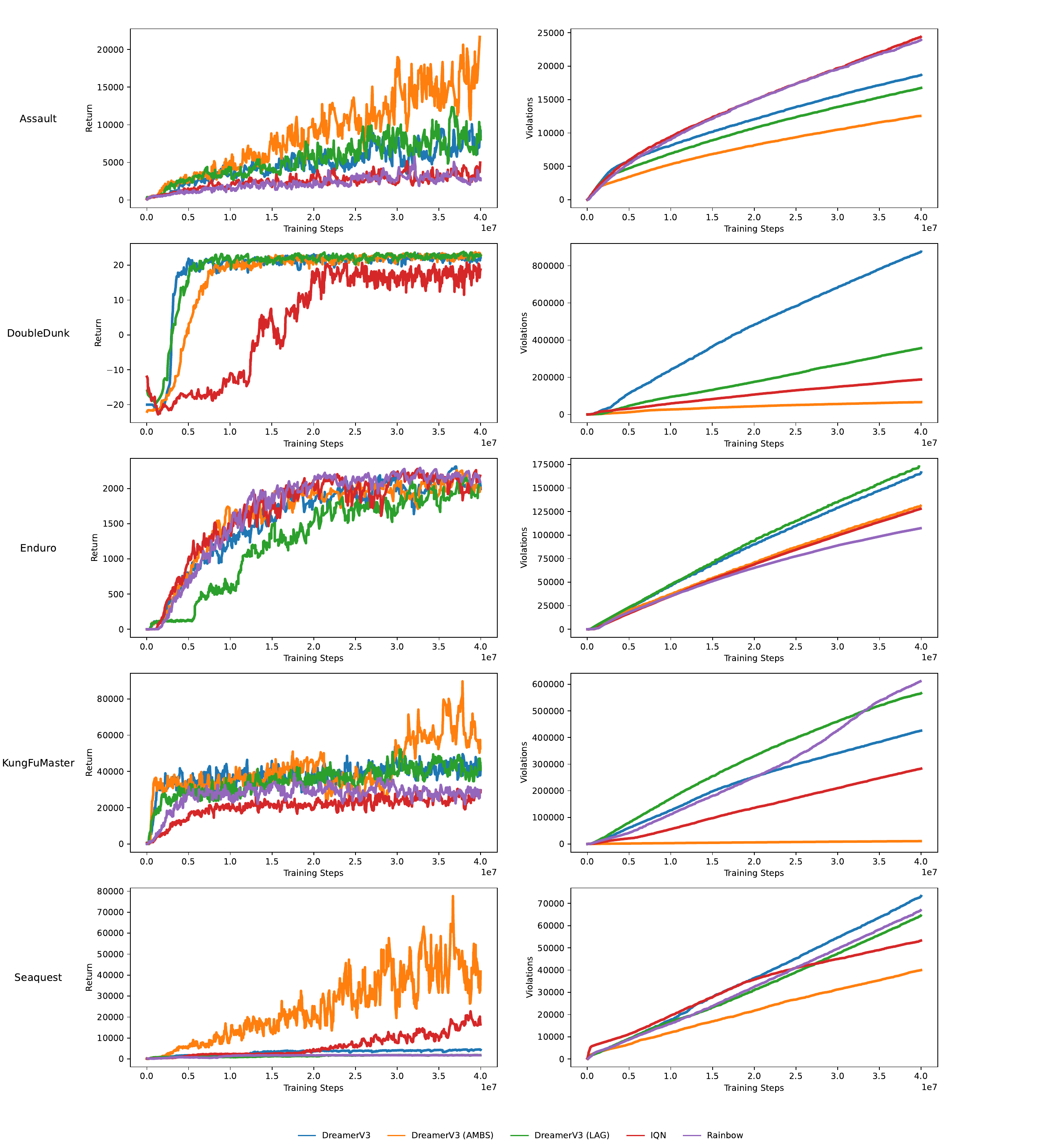}}
	\caption{Learning curves for each of the five agents on a small set of Atari games. Each line represents one run over 10M environment interactions (40M frames). The left plots represent the episode return and the right plots represent the cumulative safety-violations during training.}
	\label{fig:curves}
\end{figure*}
\vfill
\end{document}